\documentclass[letterpaper]{article} 
\usepackage{uai2020}  
\usepackage[margin=1in]{geometry}

\usepackage{times}  
\usepackage{helvet} 
\usepackage{courier}  
\usepackage[hyphens]{url}  
\usepackage{graphicx} 
\usepackage{graphicx}  
\usepackage{caption}

\usepackage[title]{appendix}
\usepackage{natbib}
\usepackage{times,amsthm,amsfonts,amsmath,amssymb,mathtools,epsfig,color,verbatim,mathtools}
\usepackage{selectp}

\usepackage{algorithm,algorithmic}
\usepackage[capitalize]{cleveref}
\crefname{ALC@unique}{line}{lines}

\newcommand{\ignore}[1]{}
\newcommand{\remove}[1]{}

\newcommand{\bbE}{\mathbb{E}}
\newcommand{\half}{\frac{1}{2}}
\newcommand{\thalf}{\tfrac{1}{2}}

\newcommand{\TV}[2]{\text{TV} [#1,#2]}
\newcommand{\Tmix}{T_{\text{mix}}}
\newcommand{\Tmixexp}{\Tmix^{\pi^E}}
\newcommand{\barTmix}{\bar T_{\text{mix}}}
\newcommand{\wt}{\widetilde}

\DeclareMathOperator*{\argmax}{arg\,max}

\newtheorem{theorem}{Theorem}

\newtheorem{lemma}[theorem]{Lemma}
\newtheorem{corollary}[theorem]{Corollary}

\theoremstyle{definition}

\newtheorem*{remark*}{Remark}
\newtheorem*{theorem*}{Theorem}
\newtheorem*{lemma*}{Lemma}

\usepackage[italic,defaultmathsizes,symbolgreek]{mathastext}
\usepackage{enumitem}
\usepackage[multiple]{footmisc}

\crefname{ALC@unique}{line}{lines}

\title{Unknown mixing times in apprenticeship and reinforcement learning}

\author{
Tom Zahavy$^{1,2}$,
Alon Cohen$^{1,2}$,
Haim Kaplan$^{1,3}$ and Yishay Mansour$^{1,3}$\thanks{Supported in part by grant  from the Israel Science Foundation}\\
$\{tomzahavy,aloncohen,haimk,mansour\} @google.com $\\
$^1$ Google AI, Tel Aviv\\
$^2$ Technion, Israel Institute of Technology\\
$^3$ Tel Aviv University \\
}
\begin{document}

\maketitle

\begin{abstract}
We derive and analyze learning algorithms for apprenticeship learning, policy evaluation, and policy gradient for average reward criteria. Existing algorithms explicitly require an upper bound on the mixing time. In contrast, we build on ideas from Markov chain theory and derive sampling algorithms that do not require such an upper bound. For these algorithms, we provide theoretical bounds on their sample-complexity and running time.
\end{abstract}

\section{INTRODUCTION}

Reinforcement Learning (RL) is an area of machine learning concerned with how agents learn long-term interactions with their environment \citep{sutton2018reinforcement}. 
The agent and the environment are modeled as a Markov Decision Process (MDP). The agent's goal is to determine a policy that maximizes her cumulative reward.
Much of the research in RL focuses on episodic or finite-horizon tasks.
When studying infinite-horizon tasks, the standard approach is to discount future rewards. Discounting serves two purposes: first, it makes the cumulative reward bounded; second, in some domains, such as economics, discounting is used to represent  ``interest'' earned on rewards. Thus an action that generates an immediate reward is preferable over one that generates the same reward in the future. Nevertheless, discounting is unsuitable in general domains. Alternatively, it is common to maximize the expected reward received in the steady-state of the Markov chain defined by the agent's policy. This is the case in many control problems: elevators, drones, climate control, etc. \citep{bertsekas2005dynamic} as well as many sequential decision-making problems such as inventory-management \citep{arrow1958studies} and queuing \citep{kelly1975networks}.

\cite{blackwell1962discrete} pioneered the study of
MDPs with average-reward criteria. He showed that the optimal
policy for the average reward is the limit of the sequence of optimal policies for discounted reward as the discount factor converges to 1.
However, it has been established that it is computationally hard to find the optimal policy when the discount factor is close to $1$.
For these reasons, Dynamic Programming (DP) algorithms were developed for average-reward criteria 
(see 
\citealp{mahadevan1996average,puterman2014markov}; for detailed surveys). 
\citet{howard1964dynamic} 
introduced the policy-iteration algorithm. Value iteration was later proposed by 
\citet{white1963dynamic}. However, these algorithms require knowledge of the state transition probabilities and are also computationally intractable.

The main challenge in deriving RL algorithms for average-reward MDPs is calculating the stationary distribution of the Markov chain induced by a given policy. This is a necessary step in evaluating the average reward of the policy. 
When the transition probabilities are known, the stationary distribution can be obtained by solving a system of linear equations. 
In the reinforcement learning setup, however, the dynamics are unknown, and practitioners tend to ``run the simulation for a sufficiently long time to obtain a good estimate'' \citep{gosavi2003simulation}. This implicitly implies that the learner knows a bound on the mixing time $\Tmix$ of the Markov chain. Indeed, model-free algorithms for the average reward with theoretical guarantees (e.g., \citealp{wang2017primal,chen2018scalable}) require an upper bound on the mixing time as an input. 

We will later see that \textbf{not knowing such a bound on the mixing time comes with an additional cost} and requires $O(\Tmix|S|)$ samples to get a single sample from the stationary distribution. Instead, one might consider learning the transition probability matrix and use it for computing the stationary distribution (i.e., model-based RL). Not surprisingly, model-based algorithms assume that the mixing time, or an upper bound on it, are known explicitly \citep{kearns2000approximate,brafman2002r}.\footnote{UCRL2 \citep{jaksch2010near} avoids using the mixing time but instead assumes knowledge of the MDP diameter (which is implicitly related to the mixing time) to guarantee an $\epsilon-$optimal policy. We emphasize that there is no need to know a bound on the diameter in the regret setting, but only when the goal is to learn an $\epsilon$-optimal policy. In this case, the learner has to know a bound on the diameter in order to bound the sample complexity.}
Moreover, even if we estimate the transition probabilities, it is not clear how to use it to obtain the average reward. In particular, the stationary distribution in the estimated model is not guaranteed to be close to the stationary distribution of the true model; an equivalent of the simulation lemma \citep{kearns2002near} for this setup does not exist. To illustrate the difficulty, consider a periodic Markov chain with states ordered in a (deterministic) cycle. Also consider a similar Markov chain, but with a probability of $\epsilon$ to remain in each state. Even though the two models are ``close'' to each other, the latter chain is ergodic and does have a stationary distribution while the former chain is periodic and therefore does not have a stationary distribution {\bf at all}.

Alternatively, one may consider estimating the mixing time (or an upper bound on it) directly, in order to use it to get samples from the stationary distribution. There are two drawbacks to this approach. First, the sample complexity for estimating the mixing time is quite significant. 
For an arbitrary ergodic Markov chain, it is possible to estimate an upper and a lower bound on the mixing time by approximating the pseudo-spectral gap and the minimal stationary probability $\pi_\star$ \citep[Theorems 12.3 and 12.4]{levin2017markov}, and estimating these quantities to within a relative error of $\epsilon$ requires $O\bigl(\Tmix^2\max\{\Tmix,\lvert S \rvert / \pi_\star\}/\epsilon^2\pi_\star\bigr)$ samples \citep{wolfer2019estimating}. 
Second, these techniques can be used to get an upper bound on the mixing time of a single policy, and not on the maximum of all the deterministic policies in an MDP. For these reasons we focus on algorithms that avoid estimating the mixing time directly.

In this work, we build on Coupling From the Past (CFTP) -- a technique from Markov chain theory that obtains unbiased samples from a Markov chain's stationary distribution \citep{propp1996exact,propp1998get}. These samples are generated \emph{without any prior knowledge on the mixing time of the Markov chain}. Intuitively, CFTP starts $|S|$ parallel simulations of the Markov chain, one from each state, at minus infinity. 
When two simulations reach the same state, they continue together as one simulation. The simulations coalesce at time zero to a single sample state, which is distributed exactly as the stationary distribution. CFTP, rather than starting at minus infinity, starts at zero and generates suffixes of increasing length of this infinite simulation until it can identify the state at which all simulations coalesce. 
The simulations are shown to coalesce, in expectation, after $O \big( \lvert S \rvert \Tmix \big)$-time. In \cref{sec:propp-wilson}, we provide an alternative, simple proof of the coalescence-time of the CFTP procedure and a matching lower bound. Additionally, we analyze the time it takes for two simulations to coalesce and show how to use this process to estimate differences of $Q$-values.

We further describe sampling-based RL algorithms for the average-reward criteria that utilize these ideas. The main advantage of our algorithms is that they do not require a bound on the associated Markov chain's mixing time. In \cref{sec:AL}, we consider apprenticeship learning and propose two sampling mechanisms to evaluate the game matrix and analyze their sample complexity. These are: using the CFTP protocol to estimate the game matrix directly at the beginning of the algorithm; querying the expert, at each step, for two trajectories to provide an unbiased estimate of the game matrix.  We also include an unbiased estimator of the policy gradient (under average reward criteria) and analyze its sample complexity. Finally, in the supplementary material (\cref{sec:PE}), we use CFTP to propose a sample-efficient data structure that allows us to get an unbiased sample from the stationary distribution of any policy in an MDP.

\subsection{PRELIMINARIES}
\label{sec:background}
In this section, we provide background on RL with \textbf{average reward criteria} (based on \citealp{puterman2014markov}), as well as on the \textbf{CFTP} algorithm \citep{propp1996exact,propp1998get} for getting unbiased samples from a stationary distribution of a Markov chain without knowing its mixing time. Background on apprenticeship learning is provided in the relevant section. 

A Markov Decision Process (MDP) consists of a set of states $S$, and a set of actions $A$. We assume that $S$ and $A$ are finite.
Associated with each action $a \in A$ is a state transition matrix $P^a$, where $P^a(x,y)$ represents the probability of moving from state $x$ to $y$ under action $a$. There is also a {\em stochastic reward} function $R : S \times A \mapsto \mathbb{R}$  where  $r(x, a)=\mathbb{E}[R(s,a)]$ is the expected reward when performing action $a$ in state $x$. A stationary deterministic policy is a mapping $\pi : S \mapsto A$ from states to actions.  Any policy induces a state transition matrix $P^\pi$, where $P^\pi(x,y) = P^{\pi(x)}(x,y)$. Thus, any policy yields a Markov chain $(S, P^\pi)$. The \textbf{stationary distribution} $\mu$ of a Markov chain with transition matrix $P$ is defined to be the probability distribution satisfying  $\mu^\top=\mu^\top P$.

We specifically study {\em ergodic MDPs} in which any policy induces an ergodic Markov chain. That is, a Markov chain which is irreducible and aperiodic \citep{levin2017markov}. Such a Markov chain converges to a unique stationary distribution independent of the starting state (for generalizations to unichain MDPs, see \citealp{puterman2014markov}). The \textbf{average reward} $\rho(\pi)$ ssociated with a particular policy $\pi$ is defined as
$\rho(\pi) = \mathbb{E}_{x \sim \mu(\pi)} \, r(x,\pi(x))$ where $\mu(\pi)$ is the stationary distribution of the Markov chain induced by $\pi$. The {\em optimal policy} is one that maximizes the average reward. The \textbf{$Q$-value} of a state-action pair given a policy $\pi$ is defined as 
\begin{equation} \label{eq:qfunctiondef}
    Q^\pi (s,a) 
    = 
    \sum_{t=0}^\infty \mathbb{E} \left\{r_t - \rho(\pi) \mid s_0 = s,a_0 = a,\pi \right\}. 
\end{equation}
    
We define the \textbf{total-variation distance} for two probability measures $P$ and $Q$ on a sample space $\Omega$ to be $\TV{P}{Q} = \sup_{A\subseteq \Omega}\lvert P(A)-Q(A)\rvert$ (which is equivalent to the $L1$ distance). Informally, this is the largest possible difference between the probabilities that the two distributions assign to the same event. 

The \textbf{mixing time} of an ergodic Markov chain with a stationary distribution $\mu$ is the smallest $t$ such that $\forall x_0,\ \TV{\Pr_t(\cdot|x_0)}{\mu} \le 1/8$, where $\Pr_t(\cdot|x_0)$ is the distribution over states after $t$ steps, starting from $x_0$. For MDP $M$, let $\Tmix^\pi$ be the mixing time of the Markov chain which $\pi$ induces in $M$, i.e., $(S,P^\pi).$
The \textbf{MDP mixing time}, $\barTmix=\max_{\pi \in \Pi} \Tmix^\pi$ is the maximal mixing time for any deterministic policy. 

The algorithms presented in this paper rely on access to a \textbf{generative model} \citep{kearns2002near}; an oracle that accepts a state-action pair $(s,a)$ and outputs a state $s'$ that is drawn from the next-state distribution $P^a(s,\cdot)$, and a sample from the reward distribution $R(s,a)$. We further assume that a sample is generated in unit time, and measure the \emph{sample complexity} of an algorithm by the number of calls it makes to the generative model.

\begin{figure*}[h]
\begin{minipage}{\textwidth}
    \begin{minipage}[b]{0.3\textwidth}
        \centering
        \includegraphics[width=\linewidth]{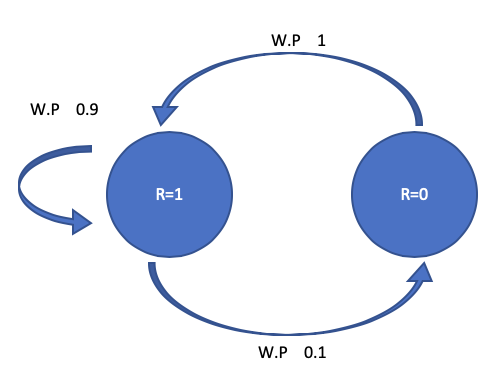}
        \captionof{figure}{Markov chain}\label{fig:mc}
    \end{minipage}
    \begin{minipage}[b]{0.34\textwidth}
        \centering
        \includegraphics[width=\linewidth]{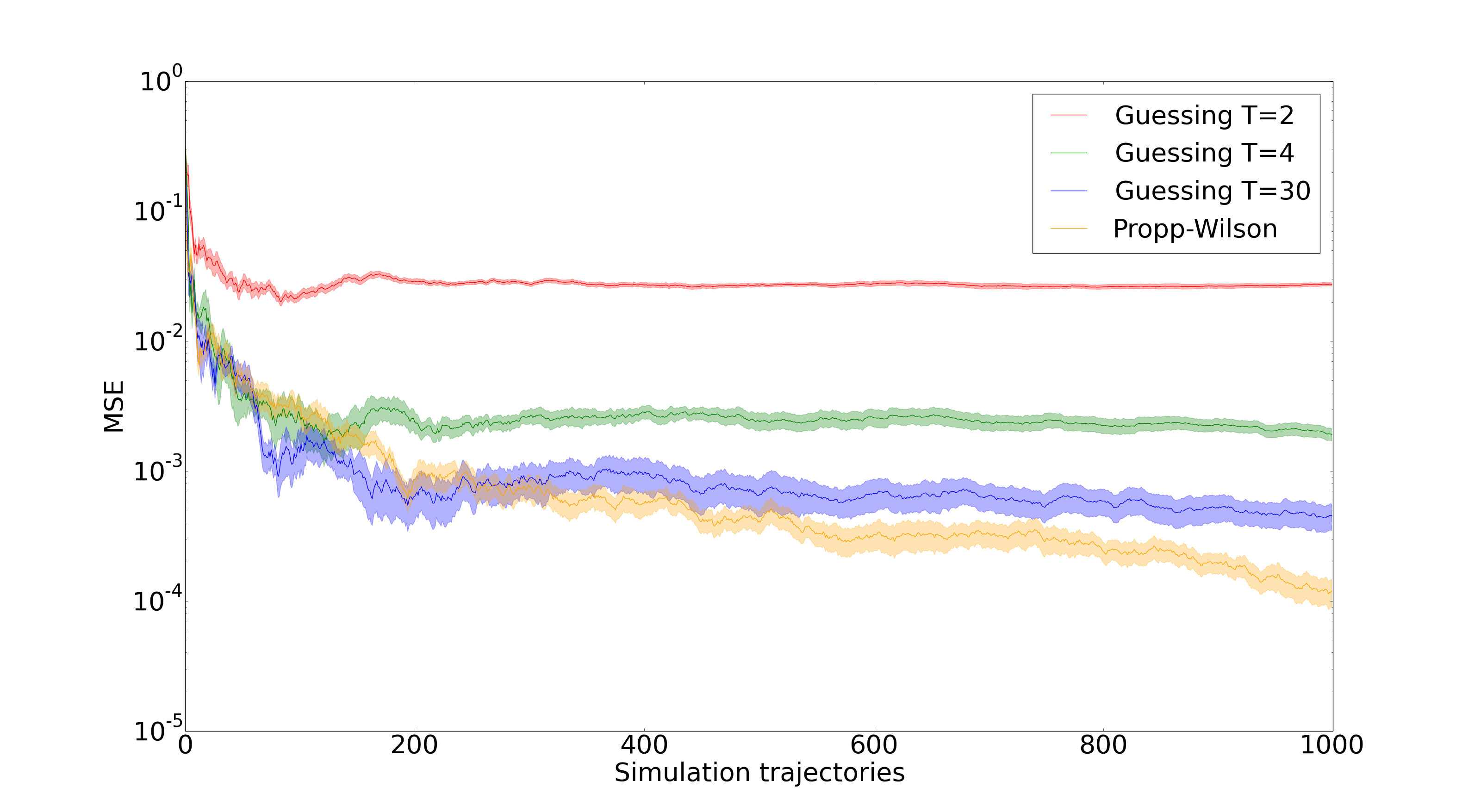}
        \captionof{figure}{MSE vs. runs }\label{fig:mse_traj}
    \end{minipage}
    \begin{minipage}[b]{0.34\textwidth}
        \centering
        \includegraphics[width=\linewidth]{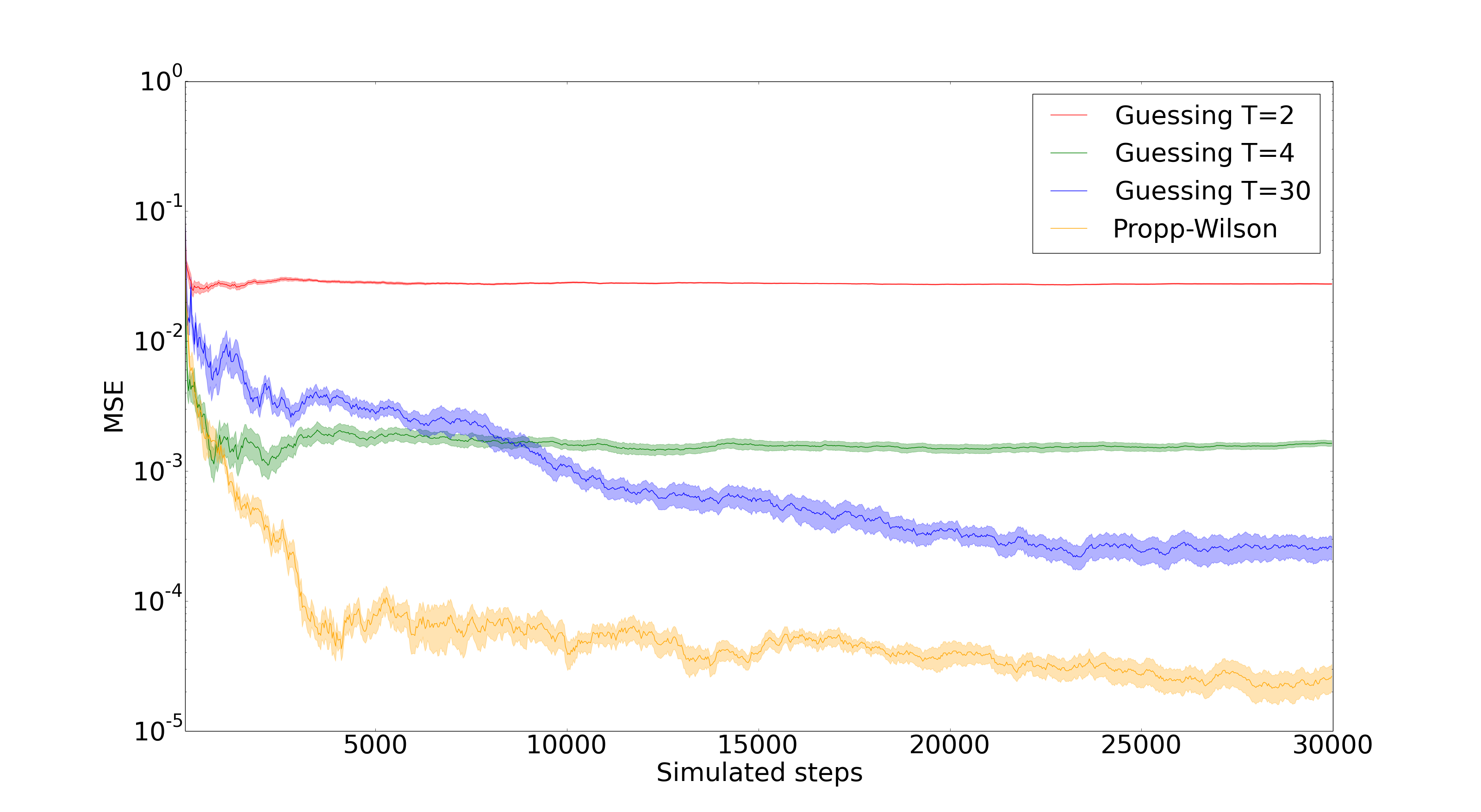}
        \captionof{figure}{MSE vs. steps}\label{fig:mse_steps}
    \end{minipage}
\end{minipage}
\end{figure*}

\textbf{Coupling From the Past (CFTP)} is a method for sampling from the stationary distribution of a Markov chain \citep{propp1996exact,propp1998get}. 
Contrary to many Markov Chain Monte-Carlo algorithms, Coupling from the past gives a perfect sample from the stationary distribution. 
Intuitively, CFTP starts $|S|$ parallel simulations of the Markov chain, one from each state, at minus infinity. 
When two simulations reach the same state, they continue together as one simulation.
The simulations coalesce at time zero to a single sample state, which is distributed exactly as the stationary distribution. CFTP rather than starting at minus infinity starts at zero and
generates suffixes of increasing length of this infinite simulation until it can identify the
state at which all simulations coalesce \citep{haggstrom2002finite}.

\begin{algorithm}[h]
\caption{Coupling from the past}
\begin{algorithmic}
\STATE $F_0\leftarrow \text{Identity Map}$, $t \leftarrow (-1)$
\REPEAT
    \STATE $t\leftarrow t+1$
    \STATE $f_{-(t+1)}\leftarrow \text{RandomMap}(P)$
    \STATE $F_{-(t+1)}\leftarrow F_{-t} \circ f_{-(t+1)}$
\UNTIL{$F_{-(t+1)}$ is constant}
\STATE Return the value into which $F_{-(t+1)}$ coalesces
\end{algorithmic}
\label{alg1}
\end{algorithm}

Consider a finite state ergodic Markov chain $M$ with state space $S$, a 
transition probability matrix $P$.
CFTP generates a sequence of mappings $F_0,F_{-1},F_{-2},\ldots$ each from $S$ to $S$, until 
the first of these mappings, say $F_{-t}$ is constant,  sending all states into the same one.
In other words, $F_{-t}$ defines  simulations from every starting state, that 
{\em coalesce} into a single state after $t$ steps.
Initially $F_0(s) = s$ for every $s\in S$. Then we generate $F_{-(t+1)}$ by drawing a random map $f_{-(t+1)} : S \mapsto S$ (denoted by RandomMap(P)) where we pick
$f_{-(t+1)}(s)$ from the next state distribution $P(s,\cdot)$ (e.g., the Markov chain dynamics) for every $s$, and compose $f_{-(t+1)}(s)$ with $F_{-t}$. 

\begin{theorem} \label{thm:cftp}
With probability 1, the CFTP protocol returns a value, which is distributed according to the Markov chain's stationary distribution. 
\end{theorem} 

See \citet{propp1996exact} for proof. Additionally, Theorem 5 in \citet{propp1998get} states that the expected value of $t$ when $F_{-t}$ coalesces is $O \big(\Tmix \lvert S \rvert \big)$.
The straightforward implementation of \cref{alg1} takes $O(\lvert S \rvert)$ time per
step for a total of $O \big(\Tmix \lvert S \rvert^2 \big)$ time.
\citet{propp1996exact} also give a cleverer implementation that takes $O \big( \Tmix \lvert S \rvert \log \lvert S \rvert \big)$ time. It uses the fact that coalescence occurs gradually and reduces the number of independent simulations as time progresses.

\subsection{EXAMPLE}

We finish this section with a motivating simulation, where we compare the CFTP procedure with a common practice of “guessing” the mixing time and running the chain for that time. While CFTP does not suffer from bias at all, the baseline methods do suffer from bias and are shown to produce errors in estimating the average reward. If the guess is too large, then these methods are highly sample-inefficient
compared to CFTP. 

Explicitly, consider the Markov reward process in \cref{fig:mc}. The initial state distribution $\mu_0$ is given by $\mu_0= (0,1)$ (starts from the right state). 
The stationary distribution is $(\frac{2}{3}, \frac{1}{3})$, the average reward is $\frac{2}{3},$ the expected coalescence time is $2$, and the mixing time is $4$.
For this chain, simulating from time $0$ forward until all chains coalesce gives a biased sample, as coalescence can only occur in the left state. 

We implemented the CFTP procedure\footnote{Accompanying code can be found in the supplementary material.} and compared it with a baseline that ``guesses'' the mixing time, $T_{\text{guess}}$. This baseline operates as follows. First, it samples an initial state from $\mu_0.$ Then, it simulates the chain for $T_{\text{guess}}$ steps. Finally, it returns the reward at the resulting state. 

We run each algorithm to obtain a sample from the stationary distribution, and use the average of these samples to estimate the average reward. For each algorithm, we report the average (over $10$ runs, reported alongside error bars) Mean Squared Error (MSE) with respect to the average reward as a function of the number of runs taken (\cref{fig:mse_traj}). We can see that underestimating the mixing time $T_{\text{guess}}=2$ (red), and using precisely the mixing time $T_{\text{guess}}=4$ (green) leads to bias in the estimation of the average reward. The latter is due to the fact that by the definition of the mixing time, we are not guaranteed to sample exactly from the stationary distribution, but only from a distribution that is close it in the total variation distance. When we overestimate the mixing time, e.g., for $T_{\text{guess}}=30$ (blue), the bias decreases significantly. Similarly, we can see that CFTP (orange) produces unbiased samples as expected. 

The advantage of CFTP (orange) becomes clearer when inspecting \cref{fig:mse_steps}. We can see the MSE as a function of the number of simulation steps. Overestimating the mixing time $T_{\text{guess}}=30$ (blue), still gives unbiased estimates but uses too many simulation steps to achieve a single sample of the reward. As a result, CFTP (orange) yields much lower MSE for the same amount of samples.

\section{SAMPLING FROM A STATIONARY DISTRIBUTION WITH UNKNOWN MIXING TIME}
\label{sec:propp-wilson}

\subsection{COALESCENCE FROM TWO STATES}
We begin this section by analyzing a simple scenario: for a Markov chain with $|S|$ states, we simultaneously start two simulations from two different states. At each time step, each simulation proceeds according to the next state distribution of the Markov chain. We are interested in bounding the time that it takes for \textbf{two} simulations to reach the same state. As we will see, this takes $O(\Tmix |S|)$ time in expectation. We begin with the following two lemmas. 
\begin{lemma}
\label{lem:closedistributionscoalesce}
Let $P$ and $Q$ be distributions on $\{1,\ldots,|S|\}$ such that $\TV{P}{Q} \le \frac{1}{4}$.
Draw $x$ from $P$ and $y$ from $Q$ independently.
Then $\Pr[x=y] \ge \tfrac{1}{2|S|}$.
\end{lemma}
\begin{proof}
Let $B = \{i \in \{1,\ldots,|S|\} : P(i) > Q(i)\}$ and note that, by definition of the total variation distance, $P(B) - Q(B) = \TV{P}{Q} \le 1/4$.
We have that
\begin{align*}
    \Pr[x = y]
    &=
    \sum_{i=1}^{|S|} P(i) Q(i) \\
    & \ge
    \sum_{i \in B} Q(i)^2
    +
    \sum_{i \in B^c} P(i)^2  \\
    & \ge
    \frac{\big(Q(B) + P(B^c) \big)^2}{|S|} \tag{Cauchy-Schwartz}~.
\end{align*}
The proof is completed by noticing that 
$
    Q(B) + P(B^c) 
    = 
    1 - \big( P(B) - Q(B) \big)
    \ge 3/4~. 
$
\end{proof}

\begin{lemma}
\label{lem:coalescenceaftermixing}
Let $i$ and $j$ be two states of an ergodic Markov chain of $|S|$ states.
Let $x$ be the state reached after making $\Tmix$ steps starting from $i$, and let $y$ be the state reached after making $\Tmix$ steps starting from $j$.
Then $\Pr[x=y] \ge \frac{1}{2|S|}$.
\end{lemma}

\begin{proof}
Let $P$ be the distribution on states after making $\Tmix$ steps starting from $i$, and let $Q$ be the distribution on states after making $\Tmix$ steps starting from $j$. Let $\mu$ be the stationary distribution of the Markov chain. Then by the definition of $\Tmix$, we have that $\TV{P}{\mu} \le 1/8$ and $\TV{Q}{\mu} \le 1/8$.
Therefore, $\TV{P}{Q} \le 1/4$, and by \cref{lem:closedistributionscoalesce} we have $\Pr[x=y] \ge \frac{1}{2|S|}$.
\end{proof}

By repeating the argument of the previous Lemma, we arrive at the following conclusion.

\begin{theorem}
\label{lem:two_traj}
Let $i$ and $j$ be two states of an ergodic Markov chain on $|S|$ states.
Suppose that two chains are run simultaneously; one starting from $i$ and the other from $j$. 
Let $T_c$ be the first time in which the chains coalesce. 
Then $T_c \le 2 \lvert S \rvert \Tmix \log(1/\delta)$ with probability at least $1-\delta$.
Moreover, 
$
    \bbE [ T_c ]
    \le 
    2 \lvert S \rvert \Tmix
$.
\end{theorem}

\begin{proof}
Let us start by sketching the proof idea.
We break time into multiples of $\Tmix$. We show that the probability that the chains do not coalesce after $\ell$ such time-multiples is at most $\left(1-\frac{1}{2|S|}\right)^\ell$.

Denote by $x_t$ and $y_t$ be the states of the two chains at time $t$.
For any $t$ and two states $i',j'$, denote by $E_{t,i',j'}$ be the event that $x_t = i'$ and $y_t = j'$.
By the Markov property and by \cref{lem:coalescenceaftermixing},
\[
    \Pr[x_{t+\Tmix} = y_{t+\Tmix} \; \vert \; E_{t,i',j'}] 
    \ge 
    \frac{1}{2|S|}~.
\]
Using the Markov property again, for any $\ell \ge 1$:
\begin{align*}
    \Pr&\Big[T_c > \ell \cdot \Tmix \mid T_c > (\ell-1)\Tmix\Big] \\
    & \le
    \Pr[x_{\ell \cdot \Tmix} \neq y_{\ell \cdot \Tmix} \mid T_c > (\ell-1)\Tmix] \\
    & =
    \sum_{i' \neq j'} 
    \Pr
    \Big[x_{\ell \cdot \Tmix} \neq y_{\ell \cdot \Tmix} \mid T_c > (\ell-1)\Tmix, \\
    &\qquad\qquad E_{(\ell-1)\Tmix,i',j'}\Big]
    \cdot
    \Pr[ E_{(\ell-1)\Tmix,i',j'}] \\
    & =
    \sum_{i' \neq j'} 
    \Pr[x_{\ell \cdot \Tmix}\neq y_{\ell \cdot \Tmix} \mid E_{(\ell-1)\Tmix,i',j'}]
    \\ & \qquad\qquad \cdot 
    \Pr[ E_{(\ell-1)\Tmix,i',j'}] \\
    & \le
    \sum_{i' \neq j'} 
    \bigg(1-\frac{1}{2|S|} \bigg)
    \cdot
    \Pr[ E_{(\ell-1)\Tmix,i',j'}] 
    \le
    1-\frac{1}{2|S|}~.
\end{align*}
Therefore,
\begin{align*}
    \Pr&[T_c > \ell \cdot \Tmix] \\
    &=
    \Pr[T_c > \ell \cdot \Tmix \mid T_c > (\ell-1)\Tmix]
     \\
    & \qquad \cdot \Pr[T_c > (\ell-1) \Tmix] \\
    &\le
    \left(1-\frac{1}{2|S|}\right) \Pr[T_c > (\ell-1) \Tmix]~,
\end{align*}
and inductively $\Pr[T_c > \ell \cdot \Tmix] \le \left(1-\frac{1}{2|S|}\right)^\ell$.

The high probability bound immediately implies a bound on the 
expected coalescence time as follows:
\begin{align*}
    \bbE [ T_c ]
    &=
    \sum_{t=0}^\infty \Pr[T_c > t] \\
    &\le
    \Tmix + \Tmix \sum_{\ell=1}^\infty \Pr[T_c > \ell \cdot \Tmix] \\
    &\le
    \Tmix + \Tmix \sum_{\ell=1}^\infty \left(1-\frac{1}{2|S|}\right)^\ell \\
    &=
    2 \lvert S \rvert \Tmix~. \qedhere
\end{align*}
\end{proof}

We finish this subsection by showing that the upper bound in \cref{lem:two_traj} is tight. 

\begin{theorem}
\label{lem:lower_bound}
There exists an ergodic Markov chain on $|S|$ states and two states $i,j$ such that the coalescence time $T_c$ of two chains running simultaneously, one starting from $i$ and the other from $j$, satisfies $\bbE [ T_c] \ge \tfrac{1}{6} \Tmix \cdot |S|$.
\end{theorem}

\begin{proof}
Let $\epsilon \in (0,1)$ and
consider a Markov chain with $|S|$ states that for each state $s$, stays at $s$ with probability $1-\epsilon$, and with probability of $\epsilon$ choose the next state uniformly at random. Then, $\Pr[s' \mid s] = (1-\epsilon){\bf 1}_{s = s'} + \frac{\epsilon}{|S|}$. It is easy to see that the stationary distribution of this chain is uniform. This means that, starting the chain at state $s_0$, with probability $\epsilon$ the distribution at any time $t \ge 1$ is uniform. Therefore, $\Pr[s_t \mid s_0] = (1-\epsilon)^t {\bf 1}_{s_t=s_0}+(1-(1-\epsilon)^t) \tfrac{1}{|S|}$.

Denote $\bar u$ the uniform distribution.
We get that
\begin{align*}
    \TV{\Pr_t[\, \cdot \mid s_0]}{\bar u} 
    &= 
    \half \sum_s \bigg\lvert \Pr[s_t = s \mid s_0] - \frac{1}{|S|} \bigg\rvert \\
    &= 
    \half \sum_s (1-\epsilon)^t \bigg\lvert {\bf 1}_{s = s_0}- \frac{1}{|S|} \bigg\rvert \\
    &= 
    \frac{|S|-1}{|S|} \cdot (1-\epsilon)^t 
    \le  e^{-\epsilon t}~.
\end{align*}
This implies that $\Tmix \le \tfrac{3}{\epsilon}$. 

Next, notice that for coalescence to happen, one of the states $i$ or $j$ must transition to a state held by the other chain, which happens with probability at most $\frac{2\epsilon}{|S|}$ via a union bound. Thus, in expectation, the time it takes for them to coalesce is $\bbE [T_c] \ge |S|/2\epsilon \ge \frac{1}{6} |S| \Tmix$.
\end{proof}

\subsection{COALESCENCE FROM $|S|$ STATES}
In the supplementary material, we provide an \emph{alternative, simple proof of the coalescence-time of the CFTP procedure}. 
The main ingredient of the proof is a generalization of the argument for bounding the coalescence-time of two chains to that of $|S|$ chains. The following theorem formalizes this.

\begin{theorem} \label{thm:proppwilson}
    Let $\mu$ be the stationary distribution of an ergodic Markov chain with $|S|$ states. 
    We run $|S|$ simulations of the chain each starting at a different state. When two or more simulations coalesce, we merge them into a single simulation.
    With probability at least $1-\delta$, all $|S|$ chains are merged after at most $512 |S| \Tmix \log(1/\delta)$ iterations.
\end{theorem}

\subsection{ESTIMATING THE DIFFERENCE IN AVERAGE REWARD OF TWO POLICIES} 
Denote the difference in average reward between two policies by $\Delta\rho(\pi,\pi') \coloneqq \rho (\pi') - \rho(\pi)$. As seen in \cref{sec:propp-wilson}, we can sample a state from the stationary distribution of $\pi$ and thereby get an unbiased estimate of $\rho(\pi)$ and similarly for $\pi'$. The difference between these estimates is an unbiased estimate of $\Delta\rho(\pi,\pi')$. However, we can also get an unbiased estimate of $\Delta\rho(\pi,\pi')$ by sampling the stationary distribution of only one of $\pi$  and $\pi'$, as we will now show. This property is useful when the sampling mechanism from one of the policies is restricted by real world constraints, e.g., in apprenticeship learning (see \cref{sec:AL_unbiased} for a concrete example). Our result builds on the following fundamental lemma  regarding  the average reward criteria (see, for example, \citealp[Lemma 5]{even2009online}): 
\begin{lemma} \label{lemma:r_diff}
$\forall \pi, \pi' \in \Pi: \Delta\rho(\pi,\pi') 
    \coloneqq 
    \rho (\pi') - \rho(\pi) 
    = 
    \bbE_{s \sim \mu({\pi'})} \left\{ Q^\pi(s,\pi'(s))-Q^{\pi}(s,\pi(s)) \right\}.
$
\end{lemma}
Recall that $Q^\pi$ is defined as in \cref{eq:qfunctiondef}.
\cref{lemma:r_diff} suggests a mechanism to estimate $\Delta\rho(\pi,\pi')$ using \cref{lem:two_traj}. We first sample a state, $s_0$ from the stationary distribution of $\pi'$. Then, we initiate two trajectories from $s_0$, the first trajectory follows $\pi$ from $s_0$ and the second trajectory takes the first action (at $s_0$) according to $\pi'$ and follows $\pi$ thereafter. We accumulate the reward achieved by each trajectory until they coalesce. The difference between these sums makes an unbiased estimate of $\Delta\rho(\pi,\pi')$.

\section{APPRENTICESHIP LEARNING}
\label{sec:AL}
Consider learning in an MDP for which the reward function is not given explicitly, but we can observe an expert demonstrating the task that we want to learn. We think of the expert as trying to maximize the average reward function that is expressible as a linear combination of known features. This is the Apprenticeship Learning (AL) problem \citep{abbeel2004apprenticeship}.
We focus on extending the Multiplicative Weights Apprenticeship Learning (MWAL) algorithm \citep{syed2008game} that was developed for the discounted reward to the average-reward criteria. Our ideas may apply to other AL algorithms as well.

\subsection{BACKGROUND}
\label{sec:al_background}
In AL, we are given an MDP dynamics $M$ that is comprised of known states $S$, actions $A$, and transition matrices $(P^a)_{a \in A},$ yet the reward function is unknown. We further assume the existence of an \emph{expert policy}, denoted by $\pi^E$, such that we are able to observe its execution in $M$. Following \citet{syed2008game}, our goal is to find a policy $\pi$ such that $\rho(\pi) \ge \rho(\pi^E) - \epsilon,$ \textbf{for any reward function}. 
To simplify the learning process, we follow \citet{syed2008game} in representing each state $s$ by a low-dimensional vector of features $\phi(s) \in [0,1]^k$. 
We consider reward functions that are linear in these features; i.e., $r(s)=w \cdot \phi(s)$, for some $ w \in \Delta^k$ where $\Delta^k$ is the $(k-1)$-dimensional probability simplex.
For compatibility with previous work, we decided to follow \citet{syed2008game} in assuming that the reward is in the probability simplex -- in other AL papers (e.g. \citealp{abbeel2004apprenticeship, zahavy2019apprenticeship}), the L2 ball was considered instead. Having the set as the simplex, combined with the use of the Hedge algorithm (see below), allowed \citet{syed2008game} to improve the complexity of the algorithm to depend logarithmically on the dimension of the features rather than polynomially as in \citet{abbeel2004apprenticeship}. 

With this feature representation, the average reward of a policy $\pi$ may be written as $\rho(\pi) = w \cdot\Phi(\pi)$ where $\Phi(\pi)$ is the expected accumulated feature vector associated with $\pi$, defined as
$
    \Phi(\pi) 
    = 
    \lim_{N\rightarrow \infty} \bbE_{\pi}  \sum\nolimits _{t=0}^{N-1} \phi(x_t)/N
$. 
Notice that similar to the average reward, this limit is not a function of the initial state when the MDP is ergodic. 

We also require the notion of a {\em mixed policy} which is a distribution over stationary deterministic policies.
Our algorithms return a mixed policy, and our analysis is with respect to this mixed policy. 
We denote by $\Psi$ the set of all mixed policies in $M$ and by $\Pi$ the set of all deterministic stationary policies in $M$.
For a mixed policy $\psi \in \Psi$ and a deterministic policy $\pi \in \Pi$, we denote by $\psi(\pi)$ the probability assigned by $\psi$ to $\pi$.
A mixed policy $\psi$ is executed by randomly selecting the policy $\pi \in \Pi$ at time $0$ with probability $\psi(\pi)$, and following $\pi$ after that. 
The definition of $\Phi$ extends naturally to mixed policies.
In terms of average reward, mixed policies cannot achieve higher average reward than deterministic policies.

We think of AL as a zero-sum game between two players, defined by the following $k \times |\Pi|$ matrix:
\begin{equation}
G(i, \pi) = \Phi(\pi)[i] - \Phi(\pi^E)[i],
\label{eq:gamematrix}
\end{equation}
where $\Phi(\pi)[i]$ is the $i$-th component of feature expectations vector $\Phi(\pi)$ for the deterministic policy $\pi$. 
Both players play a mixed policy. The row player selects a vector $w \in \Delta^k$, which is a probability distribution over the $k$ features, and the column player chooses a policy $\psi \in \Psi$. 
Then, the value of the game is defined as
\begin{align}
    v^\star 
    &= 
    \max_{\psi\in\Psi}\min_{w\in \Delta^k}  \left[ w \cdot \Phi(\psi) - w \cdot \Phi(\pi^E) \right] \nonumber \\
    &= 
    \max_{\psi\in\Psi}\min_{w\in \Delta^k}  w^\top G\psi~.
\label{eq:objective} 
\end{align}
In \cref{sec:genexpert,sec:AL_unbiased} we propose and analyze two algorithms for apprenticeship learning with the average reward criteria, based on the MWAL algorithm \citep{syed2008game}. Specifically, these algorithms learn a mixed policy $\bar \psi$ that approximately achieves the max-min value $v^*$ (defined in \cref{eq:objective}) against any $w \in \Delta^k$.
As in previous work, we assume that the dynamics are known, yet we have access to the expert policy via an expert generative model $E$. Given a state $s$, the expert generative model $E$ provides a sample from $\pi^E(s)$. In \cref{sec:genexpert} we propose an algorithm that uses Coupling From The Past (CFTP)  to estimate the feature expectations of the expert $ \Phi^E$. In \cref{sec:AL_unbiased} we propose an algorithm that queries the expert, at each step, for two trajectories to compute an unbiased estimate  $\tilde g_t$ of the column of the game matrix (\cref{eq:gamematrix}) corresponding to $\pi^{(t)}$ based on \cref{lemma:r_diff}. Both algorithms update the strategies of the min (row) and max (column) players using standard RL methods as follows.

\textbf{(i)} Given a $\min$ player strategy $w$, find $\argmax_{\pi \in \Pi} G(w, \pi) = \sum_{i=1}^k w(i) G(i,\pi)$. This step is equivalent to finding the optimal policy in an MDP with a known reward and can be solved for example with Value Iteration or Policy Iteration. 
\textbf{(ii)} Given a $\max$ player strategy $\pi$, the min player maintains  a probability vector $w \in \Delta^k$ giving a weight to each row (feature). To update the weights, we estimate $G(i,\pi)$ for each $i \in \{1,\ldots,k\}$ and the policy $\pi$ of the max player.
The algorithms in \cref{sec:genexpert,sec:AL_unbiased}  differ in the way they estimate these $G(i,\pi)$'s.
In \cref{sec:genexpert}, we estimate the features expectations of the expert once before we start. Then, in each iteration, we evaluate the feature expectations of $\pi$ by solving a system of linear equations using the known dynamics. Then we estimate  $G(i,\pi)$ by subtracting the features expectations of $\pi$ from the estimates of the features expectations of $\pi^E$. We note that we can also handle the case where this step (and the PI step) is inaccurate. In this case, the representation error would appear in the bounds of the theorems below. Importantly, the complexity of both steps in our algorithms grows with the size of the MDP, but not with the size of the game matrix. In \cref{sec:AL_unbiased} we take a different approach and estimate the difference directly by generating two trajectories of the expert from two particular states.

\subsection{ESTIMATING THE FEATURE EXPECTATIONS OF THE EXPERT} \label{sec:genexpert}



The algorithm of this section uses CFTP  to obtain samples from the expert's stationary distribution
and uses them to estimate $\wt \Phi^E$--the expert's feature expectations.
See 
\cref{alg:mwal}, line 3.
Obtaining each of these samples requires $\Theta(|S|\Tmixexp)$ calls to the generative model (\cref{thm:proppwilson}), totaling at $O(|S|\Tmixexp\cdot m)$ calls overall. The number of samples $m$ is taken to be large enough so that the estimate $\wt \Phi^E$ is $\epsilon$-accurate. The following theorem describes the sample complexity of \cref{alg:mwal}. 

\begin{theorem}
\label{theorem:MW_knownp}
Assume we run \cref{alg:mwal} for $T =  \frac{144}{\epsilon^2}\log k $ iterations,
using $m = \frac{18}{\epsilon^2}\log (2k/\delta) $ samples from $\mu(\pi^E)$.
Let $\bar\psi$ be the mixed policy returned by the algorithm. 
Let $v^\star$ be the game value as in \cref{eq:objective}.
Then, we have that
$
    \rho(\bar\psi) - \rho(\pi^E)
    \ge 
    v^\star - \epsilon
$ 
with probability at least $1-\delta$,
where $\rho$ is any average reward of the form 
$r(s) = w\cdot \phi(s)$ where $w\in\Delta_k$.
\end{theorem}

Note that \cref{theorem:MW_knownp} is  similar to Theorem 2 of \citet{syed2008game}. The main difference is that our result applies to the average-reward criteria, and we evaluate the expert using samples of its stationary distribution instead of using trajectories of finite length (which are biased). This simplifies the analysis and gives tighter bounds. As a comparison, the iteration complexity of MWAL is $T = O ( \frac{\log (k)}{\epsilon^ 2 (1-\gamma)^2 } ),$ which is also logarithmic in $k$ and linear in $1/\epsilon^2$ but depends in the discount factor. In the discounted case, a complete trajectory is required in order to have a single unbiased estimate of the feature expectations. In the average reward case, on the other hand, a single sample from the stationary distribution suffices to create an unbiased estimate of the feature expectations, and therefore the iteration complexity does not depend on the trajectory length. More details can be found in the proof (\cref{sec:thm8}).

\begin{algorithm}[h]
\caption{MWAL for average reward criteria}
\label{alg:mwal}
\begin{algorithmic}[1]
\STATE \textbf{Given}: MDP dynamics $M$; generative model of the expert policy $E$; feature dimension $k$; number of iterations $T$; $m$ the number of samples from $\mu(\pi^E). $
\STATE Let $\beta = \sqrt{\frac{\log k}{T}}$ (learning rate)
\STATE \textbf{Sampling}: Use the CFTP protocol with $E$ and $M,$ to obtain $m$ samples $\{ \phi(s_i)\}_{i=1}^m$ s.t. $s_i$ are i.i.d random variables and $s_i \sim \mu(\pi^E).$ Let $\wt \Phi^E = \frac{1}{m}\sum_{i=1}^m \phi(s_i)$. \label{ln:sampling}
\STATE Initialize $W^{(1)}(i) = 1$, for $i = 1,\ldots,k$.
\FOR{$t = 1,\ldots,T$}
    \STATE Set $w^{(t)}(i) = \frac{W^{(t)}(i)}{\sum_{i=1}^k W^{(t)}(i)}$, for $i = 1,\ldots,k$.
    \STATE Compute an optimal policy $\pi^{(t)}$ for $M$ with respect to reward function $r^{(t)}(s) = w^{(t)}\cdot\phi(s)$. 
    \FOR{$i = 1,\ldots,k$}
        \STATE Set $\tilde g_t(i) = \left(\Phi(\pi^{(t)})[i] - \wt \Phi^E[i] + 1\right) / 2$. 
        \STATE $W^{(t+1)}(i) = W^{(t)}(i) \cdot \exp\left(-\beta \tilde g_t (i) \right)$. 
    \ENDFOR
\ENDFOR
\STATE Post-processing: Return the mixed policy $\bar\psi$ that assigns probability $\frac{1}{T}$ to $\pi^{(t)}$, for all $t \in \{1,\ldots, T\}$.
\end{algorithmic}
\end{algorithm}

\begin{remark*}
Recall that the expert policy may be stochastic. At first glance, it may be tempting to try to estimate the expert policy directly.
However, note that $\phi(\pi^E)$ is an expectation over the expert's stationary distribution.
Even if we do manage to estimate the expert's policy to $\epsilon$-accuracy in each state, the small error in the estimated policy may entail a significant error in its stationary distribution. In fact, this error might be as large as $\Omega(\Tmixexp \epsilon)$. 
In particular, there is no sample size which is oblivious to  $\Tmixexp$ and guarantees an $\epsilon $ bounded error.
\end{remark*}

\subsection{ESTIMATING THE GAME MATRIX DIRECTLY}
\label{sec:AL_unbiased}
In the previous section, we introduced an algorithm that uses the CFTP protocol to sample the expert's stationary distribution without any knowledge of the corresponding Markov chain's mixing time. However, this mechanism required to query the expert
for a long trajectory starting from {\bf every} state to obtain a single sample from the stationary distribution. This may be tedious for the expert in practice, in particular in domains with large state spaces.

To relax this requirement, Algorithm \ref{alg:mwal2} uses a different sampling mechanism that is \textbf{not} estimating $\Phi(\pi^E)$ at the beginning of the algorithm. Instead, Algorithm \ref{alg:mwal2} queries the expert for two trajectories \textbf{at each step} to generate an unbiased estimate  $g_t$ of a particular column of the game matrix. 
To obtain the estimate $g_t$ (Algorithm \ref{alg:mwal2}, line 7), we use the sampling mechanism developed in \cref{sec:propp-wilson} for evaluating the difference in the average reward of  two policies $\Delta \rho (\pi,\pi')$ (\cref{lemma:r_diff}). 
Specifically, we start by sampling a state $s_0$ from the stationary distribution of $\pi^{(t)}$. Since $\pi^{(t)}$ and the dynamics are known, the stationary distribution of $\pi^{(t)}$  can be computed by solving a system of linear equations. Next, we initiate two trajectories from $s_0$; the first trajectory follows the expert policy from $s_0$ and the second trajectory takes the first action (at $s_0$) according to $\pi^{(t)}$ and follows the expert after that. We accumulate the features $\phi(s)$ along the trajectories until they coalesce. The difference between these sums gives the unbiased estimate $g_t$ of $G(\cdot,\pi^{(t)}) = \Phi(\pi^{(t)}) - \Phi(\pi^E)$ (the column of the game matrix $G$ corresponding to  $\pi^{(t)}$).

\begin{algorithm}[h]
\caption{MWAL with generative differences}
\label{alg:mwal2}
\begin{algorithmic}[1]
\STATE \textbf{Given}: MDP dynamics $M$; generative model of the expert policy $E$; feature dimension $k$; number of iterations $T$; parameter $\delta$; parameter $b$.
\STATE Let $\beta = \sqrt{\frac{\log k}{T}},$ $B=b\log(2Tk/\delta)$ 
\STATE Initialize $W^{(1)}(i) = 1$, for $i = 1,\ldots,k$.
\FOR{$t = 1,\ldots,T$}
    \STATE Set $w^{(t)}(i) = \frac{W^{(t)}(i)}{\sum_{i=1}^k W^{(t)}(i)}$, for $i = 1,\ldots,k$.
    \STATE Compute an optimal policy $\pi^{(t)}$ for $M$ w.r.t \hskip0.5em reward function $r^{(t)}(s) = w^{(t)}\cdot\phi(s)$. 
    \STATE \textbf{Sample} $g_t$ s.t. $\bbE [g_t(i)] = G(i, \pi^{(t)}),$ $\forall i= 1,\ldots,k$ 
    \FOR{$i = 1,\ldots,k$}
        \STATE Set $\tilde g_t(i) = \left(g_t(i) + B\right) / 2B$. 
        \STATE $W^{(t+1)}(i) = W^{(t)}(i) \cdot \exp\left( - \beta \tilde g_t (i) \right)$. 
    \ENDFOR
\ENDFOR
\STATE Post-processing: Return the mixed policy $\bar\psi$ that assigns probability $\frac{1}{T}$ to $\pi^{(t)}$, for all $t \in \{1,\ldots, T\}$.
\end{algorithmic}
\end{algorithm}

\cref{MW_unbiased_loss} below presents the sample complexity of this approach as a function of $b$: a parameter that bounds the estimates $g_t$ with high probability. Concretely, we assume that for any $\ell > 0$, $\Pr[ \| g_t\|_\infty > \ell \cdot b] \le e^{-\ell}$. In view of \cref{lem:two_traj}, $b$ is always upper bounded by $\lvert S \rvert \Tmixexp.$ But, we believe that it can be much smaller in practice and that there exists many cases where $b$ can be known a-priori due to the structure of the reward function. For example, consider an MDP with a $p-$sparse reward function, i.e., the reward (and the feature vector in these states) is not zero in at most $p$ states. While it might take a long time for two trajectories to coalesce, the difference in the reward between these trajectories can be upper bounded using the sparsity degree $p$ of the reward. Concretely, consider an MDP with the following dynamics: $P(s_i,s_{i+1})=1, \forall i\in[1,..n-1], P(s_n,s_{n})=1-\epsilon,P(s_n,s_1=\epsilon )$, and a $p-$sparse reward function. For $\epsilon \ll 1/n$, the trajectories will coalesce at $s_n$ (with high probability), and we have that for any $\ell > 0,$  $\Pr[ \| g_t\|_\infty > \ell \cdot p] \le e^{-\ell}.$

\begin{theorem}
\label{MW_unbiased_loss}
Assume we run \cref{alg:mwal2} for $T $ iterations, and
there exists a parameter $b$, such that for any $\ell$, $\Pr(\| g_t\|_\infty \ge \ell\cdot b) \le e^{-\ell}$.
 Let $\bar\psi$ be the mixed policy returned by the algorithm. 
Let $v^\star$ be the game value as in \cref{eq:objective}.
Then, there exists a constant $c$ such that for $T \ge c B\log^2 B$ where $B=\frac{b^2}{\epsilon^2}\log^3 k\log^2(1/\delta)$, we have that
$
    \rho(\bar\psi) - \rho(\pi^E)
    \ge 
    v^\star - \epsilon
$ 
with probability at least $1-\delta$,
where $\rho$ is the average of any reward of the form 
$r(s) = w\cdot \phi(s)$ where $w\in\Delta_k$. 
\end{theorem}

The key difference from the proof of \cref{theorem:MW_knownp} is in refining the original analysis to incorporate the variance of the estimates $g_t$ into the algorithm's iteration complexity. The proof is found in \cref{sec:thm9}.

\section{POLICY GRADIENT}
\label{sec:PG}
Consider the problem of finding the best policy in an MDP from the set of all policies that are parameterized by a vector $\theta$. \citet{sutton2000policy} proposed a variant of Policy Iteration that uses the unbiased estimate of the policy gradient and guaranteed that it converges to a locally optimal policy. We now describe a sampling mechanism that achieves such an unbiased sample, resulting in a much simpler algorithm than the biased policy gradients algorithm of \citep{baxter2001infinite,marbach2001simulation}.

The Policy Gradient Theorem \citep{sutton2000policy}, states that for the average-reward criteria,
$$ 
    \frac{\partial\rho}{\partial\theta} 
    = 
    \mathbb{E}_{s\sim \mu(\pi)} 
    \mathbb{E}_{a\sim\pi(s)} \frac{\partial\log\pi(s,a)}{\partial\theta} Q^\pi(s,a)~,
$$
where $
Q^\pi(s,a) = \sum_{t=1}^\infty \mathbb{E} (r_t - \rho (\pi) | s_0 = s, a_0 = a, \pi).
$
We produce an unbiased estimate of the policy gradient similarly to evaluating the reward difference between policies described in Section~\ref{sec:propp-wilson}. Specifically,
we do the following: (1) use the CFTP method to get unbiased sample $s \sim \mu(\pi)$ from the stationary distribution of $\pi$; (2) sample $a' \sim \pi(s)$; (3) initiate two trajectories from $s$. The first trajectory starts by taking action $a$ (the action we want to estimate the Q function at), and the second starts by taking $a'$.  Even if these actions are the same, they would not necessarily lead to the same state as the environment is stochastic. After the first action is taken, both trajectories follow $\pi$ until coalescence. The difference of cumulative rewards between the two trajectories forms an unbiased estimate of the $Q$-value. To construct the estimate of the gradient, we multiply the $Q$-value estimate by the derivative of $\log \pi(s,a)$ at $\theta$.

\section{DISCUSSION}
We derived and analyzed reinforcement learning algorithms for average reward criteria. Existing algorithms explicitly require an upper bound on the mixing time. In contrast, we leveraged the CFTP protocol and derived sampling algorithms that \textbf{do not require such an upper bound}. For these algorithms, we provided theoretical bounds on their sample-complexity and running time. Finally, we offered an alternative, simpler proof for the correctness of CFTP. As CFTP is a twenty-year-old protocol, we hope that our proof will make it more accessible to the RL community.

\bibliographystyle{icml2020}
\bibliography{paper_bib}
\clearpage
\onecolumn 
\begin{appendices}

\section{DYNAMIC DATA STRUCTURE FOR POLICY EVALUATION}
\label{sec:PE}

In this section, we describe a sample-efficient data structure that allows us to get an unbiased sample from the stationary distribution of any deterministic policy.

For motivation, consider estimating the average reward of a single policy $\pi$, for which we can use CFTP to get a sample $s$ from the stationary distribution $\mu(\pi)$ (\cref{thm:cftp}). Then we sample $R(s,a)$ and get an unbiased estimate of $\rho(\pi)$.
To estimate $\rho(\pi)$ to an accuracy of $\epsilon$ with confidence $\delta$ we average $O\bigl(\frac{1}{\epsilon^2}\log\left(1/\delta\right)\bigr)$ such samples.

By Theorem  \ref{thm:proppwilson} it takes $O \bigl( \Tmix^\pi \lvert S \rvert \bigr)$ to get one sample from $\mu(\pi)$, so in total we would need 
$O\bigl(\frac{1}{\epsilon^2}\log\left(1/\delta\right)\Tmix^\pi \lvert S \rvert \bigr)$ to estimate the average reward of each single policy $\pi$ to an accuracy of $\epsilon$ with confidence $\delta$. Naively,
to estimate the average reward of each of the  $\lvert A \rvert^{\lvert S \rvert}$ policies separately, we need a fresh set of samples for every
policy for a total of
$O\bigl(\frac{1}{\epsilon^2}\log\left(1/\delta\right)\sum_\pi{(\Tmix^\pi)} \lvert S \rvert)$ samples.

Instead, we propose
to allow estimates of different policies to share samples by maintaining
a matrix $D$  that we use to estimate the reward of any policy $\pi$.
Each column of $D$ corresponds to a state-action pair. Each row contains,
a sample of $R(s,a)$ and a sample 
$s' \sim P^a(s,\cdot)$ obtained using the MDP's generative model, for each state-action pair $(s,a)$.
We get an unbiased estimate of $\rho(\pi)$ for some policy $\pi$ as follows.
We focus on the columns of $D$ that represent pairs $(s,\pi(s))$.
The restriction of each row to these columns gives a random mapping from states to next states in the Markov chain induced by $\pi$.
We now use these samples to run CFTP on this Markov chain, where row $t$ gives the random mapping $f_{-t}$ of Algorithm \ref{alg1}.
CFTP gives a sample $s$ from $\mu(\pi)$ and then
from the entry in $D$ to which all simulations coalesce. The sample $R(s,\pi(s))$ is an unbiased sample of $\rho(\pi)$.

The matrix $D$ is empty at the beginning and we add rows to it on demand when we  estimate 
$\rho(\pi)$ for a policy $\pi$. To analyze the expected size of $D$, observe that $n = O(\Tmix^\pi \lvert S \rvert)$ rows are needed to get an unbiased
estimate of $\rho(\pi)$ (\cref{thm:proppwilson}). This, in turn requires $O(n|S||A|)$ calls to the generative model (to fill these $n$ rows of $D$).
 
To get unbiased samples from the Markov chain of a different policy $\pi'$, we use the rows of $D$ that were already generated for estimates of previous policies (restricted to a different set of columns). If there are not enough rows for \cref{alg1} to give a sample from $\mu(\pi')$, we add rows to $D$ until coalescence occurs.
To get $\epsilon$-approximate estimates with confidence $1-\delta$ for a set of policies $\Pi$ we need to maintain $O\left(\frac{1}{\epsilon^2} \log \left(|\Pi|/\delta\right)\right)$ independent copies of $D$ and average the unbiased estimates that they return. 

In summary, the number of rows that we add to $D$ depends on the largest mixing time of a policy, which we evaluate and on the approximation guarantee $\epsilon$ and confidence requirement $\delta$.
\cref{Theo:policy_eval} states the overall sample complexity of $D$ for evaluating the reward of a set of policies $\Pi$.\footnote{Notice that while the sample complexity depends on the maximum mixing time of a policy in $\Pi$ our algorithm does not need to know it.}

\begin{theorem} 
Assume that 
we use $D$  as described above to estimate $\rho(\pi)$ for every $\pi$ in a set of policies $\Pi$ such that
 with probability at least $1-\delta$, it holds simultaneously for all $\pi \in \Pi$ that $|\tilde \rho(\pi) - \rho(\pi)|\le \epsilon$
 where $\tilde \rho(\pi)$ is our estimate of $\rho(\pi)$.
 Then the 
  expected number of calls made to the generative model is $O \bigl( n \lvert S \rvert^2 \lvert A \rvert \barTmix \bigr)$, where 
  $n = \tfrac{1}{\epsilon^2} \log \big( \lvert \Pi \rvert / \delta \big)$ and 
  $\barTmix$ is an upper bound on the mixing time of all policies in $\Pi.$
\label{Theo:policy_eval}
\end{theorem}

Notice that when $\Pi$ is the set of all deterministic policies, then $|\Pi|=|A|^{|S|}$ and the sample complexity is $O \bigl(\frac{1}{\epsilon^2}|S|^3|A|\log\left(|A|/\delta\right)\barTmix\bigr)$. This reveals the advantage of using the dynamic data structure:
We can estimate the reward of exponentially many policies with a polynomial number of samples.

\begin{proof}
Let $Z_i = \rho(\pi)_i - \rho(\pi)$. Then $\mathbb{E}(Z_i) = 0$, and $|Z_i| \le 1$. Chernhoff bound implies that given independent random variables $Z_1, . . . , Z_n$ where $|Z_i| \le 1$, $\mathbb{E}Z_i = 0$, then $\text{Prob}(\sum^n_{i=1} Z_i > a) < e^{-\frac{a^2}{2n}}$. Hence, Chernoff bound implies that (for any $n$) $\text{Prob}(\sum^{n}_{i=1} Z_i > \frac{n\epsilon}{2}) < e^{-\frac{\epsilon ^2 n}{8}}$. This implies that $\text{Prob}(\sum_{i=1}^n\left( \rho(\pi)_i - \rho(\pi) \right) > \frac{n\epsilon}{2}) = \text{Prob}(\wt \rho(\pi) - \rho(\pi) > \frac{\epsilon}{2}) < e^{-\frac{\epsilon ^2 n}{8}}$. 

Similarly, we can define $Z_i =  \rho(\pi) - \wt\rho(\pi)$ and get that $\text{Prob}(\rho(\pi) - \wt \rho(\pi)  > \frac{\epsilon}{2}) < e^{-\frac{\epsilon ^2 n}{8}}$. Hence, we get that $\text{Prob}(|- \rho(\pi)| > \frac{\epsilon}{2}) < 2e^{-\frac{\epsilon ^2 n}{8}}$. So far we have restricted our attention to a fixed policy $\pi$. Using the so-called union bound , we have that the probability that some $\pi \in \Pi$ deviates by more than $\frac{\epsilon}{2}$ is bounded by $2me^{-\frac{\epsilon ^2 n}{8}}$. Plugging $n=-\frac{8}{\epsilon^2}\text{ln}\left(\frac{\delta}{2m}\right)$ concludes our proof. 
\end{proof}

\section{A simple proof of the Propp-Wilson theorem}

\begin{theorem*}[Restatement of \cref{thm:proppwilson}] 
    Let $\mu$ be the stationary distribution of an ergodic Markov chain with $|S|$ states. 
    We run $|S|$ simulations of the chain, each starting at a different state. When two or more simulations coalesce, we merge them into a single simulation.
    With probability at least $1-\delta$, all $|S|$ chains are merged after at most $512 |S| \Tmix \log(1/\delta)$ iterations.
\end{theorem*}

The proof of this Theorem is as follows. 

We split time into blocks of size $\Tmix$. By the end of the first block, each chain is distributed with some distribution $P$ for which $\TV{P}{\mu} \le 1/8$. We check which of the chains arrive at the same state; chains that do---coalesce. Next, we utilize the Markov property and condition on the states arrived by the chains. On this event, we continue simulating the chains until the end of the next block and continue in this manner.
    
This is analogous to the following balls-and-bins process. We have $|S|$ balls and $|S|$ bins where the balls simulate the chains, and the bins simulate the states. Each ball $j$ has a distribution $P_j$ over the bins where $\TV{P_j}{\mu} \le 1/8$. We throw the balls into the bins. After that, take one ball out of each nonempty bin and discard the remaining balls. We throw the balls taken out again, and repeat this process until we are left with a single ball. 

The following Lemma shows a bound on the expected number of balls removed at each iteration.

\begin{lemma}
    \label{lem:singlephase}
    Assume $2 \le m \le n$.
    Suppose each ball $j = 1,\ldots,m$ is distributed by $P_j$, and that there is a distribution $\mu$ such that $\TV{P_j}{\mu} \le 1/8$ for all $j=1,\ldots,m$.
    Then the expected number of nonempty bins is at most $m - m^2 / 256 n$.
\end{lemma}

\begin{proof}
    Suppose we throw the balls one by one into the bins.
    We say that a ball \emph{coalesces} if it is thrown into a nonempty bin. Thus, the number of nonempty bins by the end of the process is exactly $m$ minus the total number of coalescences. Hence we proceed by lower bounding the expected number of coalescences.
    
    We split the balls into two disjoint groups of (roughly) equal sizes: $M$ of size $\lceil m/2 \rceil$ and $M^c$ of size $\lfloor m/2 \rfloor$. We first throw the balls in $M$ and \emph{thereafter} the balls in $M^c$. The total number of coalescences is, therefore, lower bounded by the number of coalescences that occur between the balls in $M^c$ and those in $M$. We continue by showing that the probability of a ball in $M^c$ to coalesce with any ball in $M$ is at least $m / 64 n$.
    Then, the expected total number coalescences is at least 
    \[
        \lvert M^c \rvert \cdot \frac{m}{64n} 
        =
        \bigg\lfloor \frac{m}{2} \bigg\rfloor \cdot \frac{m}{64n}
        \ge
        \frac{m^2}{256n}~,
    \]
    since $m \ge 2$.
    
    Indeed, let $Q_k$ denote the probability distribution over the bins of some ball $k \in M^c$, and $P_j$ denote the probability distribution of $j \in M$. 
    Ball $k$ coalesces with a ball in $M$ if it was thrown into a bin that was not empty after the first phase. Thus, we split the bins into two groups: those who are likely to be empty after the first phase, and those that are not.
    Let $S = \{i \in [n] : \sum_{j \in M} P_j(i) \le 1 \}$. The proof continues differently for two cases: either $k$ is likely to be thrown into into a bin in $S$ or not.
    If $Q_k(S) \ge 1/2$, \cref{lem:coalescencehardcase} below states that the probability of a coalescence is at least
    \[
        \frac{\lvert M \rvert}{32 n}
        \ge
        \frac{m}{64n}
    \]
    since $\lvert M \rvert \ge m/2$.
    If $Q_k(S) < 1/2$, \cref{lem:coalescenceeasycase} found below implies that the probability of a coalescence is at least
    $
        1/4 \ge m / 64n
    $.
\end{proof}

Having proven \cref{lem:singlephase}, it remains to use it to show that the expected number of iterations is $O(n)$, which we prove in the following Lemma.

\begin{lemma}
\label{lem:ballsandbinstime}
Suppose we have $|S|$ balls distributed by $P_j$, $j=1,\ldots,n$, where $\TV{P_j}{\mu} \le 1/8$. Throw the balls into the bins. Thereafter, take one ball out of each nonempty bin and throw these balls again. Let $m_t$ be the number of balls remaining at iteration $t$, where $m_0 = n$. Then, $\bbE m_t \le 256n / t$.
\end{lemma}

\begin{proof}
We show that $\bbE m_t \le 256n / t$. 
Denote $\bar m_s = \bbE m_s$. By Jensen's inequality and \cref{lem:singlephase},
\[
    \bar m_s
    \le 
    \bbE m_{s-1} - \frac{\bbE m_{s-1}^2}{256n}
    \le
    \bar m_{s-1} - \frac{\bar m_{s-1}^2}{256 n}
    \le
    \bar m_{s-1} - \frac{\bar m_{s-1} \bar m_{s}}{256 n}
\]
as $\bar m_s \le \bar m_{s-1}$ in particular. Dividing both sides of the inequality by $\bar m_s \bar m_{s-1}$ gives
\[
    \frac{1}{\bar m_{s-1}} \le \frac{1}{\bar m_s} - \frac{1}{256 n}~.
\]
By summing over $s=1,\ldots,t$ we obtain
\[
    \frac{1}{\bar m_0} \le \frac{1}{\bar m_t} - \frac{t}{256 n}~.
\]
Finally, we use $\bar m_0 \ge 0$ and rearrange the inequality above to gets the claim of the Lemma.
\end{proof}
With the Lemma at hand, the proof of \cref{thm:proppwilson} is as follows. After $512n$ iterations, the process is complete with probability at least $\thalf$ by Markov's inequality. If it is not done, we condition on the remaining set of balls and run the process for another $512n$ iterations. Once again, the process is complete with probability at least $\thalf$.
Repeating this procedure for $\log_2(1/\delta)$ times, we conclude that the procedure is complete with probability at least $1-\delta$. This finishes the proof of Theorem \ref{thm:proppwilson}. \qed

We finish this Section by proving \cref{lem:coalescencehardcase,lem:coalescenceeasycase}. We begin with \cref{lem:coalesceonsubset} that is needed for the proof of \cref{lem:coalescencehardcase}.

\begin{lemma}
\label{lem:coalesceonsubset}
    Let $P$ and $Q$ be two distribution on $\{1,\ldots,n\}$ such that $\TV{P}{Q} \le 1/4$. 
    Let $S \subseteq [n]$ be such that $Q(S) \ge \thalf$.
    Let $x$ be an element drawn from $P$ and let $y$ be an element draws from $Q$ such that $x$ and $y$ are independent.
    Then $\Pr[\exists i \in S \, : \, x = y = i] \ge 1/16n$.
\end{lemma}

\begin{proof}
    Define $B = \{i \; : \; P(i) > Q(i) \}$. Then
    \begin{align*}
        \Pr [\exists i \in S \, : \, x = y = i] 
        &= 
        \sum_{i \in S} P(i) Q(i) \\
        &= 
        \sum_{i \in S \cap B} P(i) Q(i)
        +
        \sum_{i \in S \cap B^c} P(i) Q(i) \\
        &\ge
        \sum_{i \in S \cap B} Q^2(i)
        +
        \sum_{i \in S \cap B^c} P^2(i) \\
        &\ge
        \frac{\big(Q(S \cap B) 
        +
        P(S \cap B^c) \big)^2}{\lvert S \rvert} \\
        &=
        \frac{\big(Q(S) - \big(Q(S \cap B^c) 
        -  
        P(S \cap B^c) \big) \big)^2}{\lvert S \rvert} \\
        &\ge
        \frac{\big(Q(S) - \TV{P}{Q} \big)^2}{n} \\
        &\ge
        \frac{\big(1/2 - 1/4 \big)^2}{n} \\
        &=
        \frac{1}{16n}~,
    \end{align*}
    where the fourth derivation follows from the Cauchy-Schwarz inequality.
\end{proof}

\begin{lemma}
    \label{lem:coalescencehardcase}
    Suppose we first throw a set of balls $j \in M$ with probability distributions $P_j$. Thereafter, we throw an additional ball with probability distribution $Q$ such that 
    $\TV{P_j}{Q}
     \le 1/4$ for every $j \in M$. Additionally, assume that
     $Q(S) \ge 1/2$ for $S = \{i \in [n] : \sum_{j \in M} P_j(i) \le 1 \}$. 
    Then, the probability that $Q$ is thrown into a nonempty bin is at least $\lvert M \rvert / 32 n$.
\end{lemma}

\begin{proof}
    The probability that bin $i \in S$ is nonempty is
    \[
        1 - \prod_{j \in M} \big(1 - P_j(i) \big)
        \ge
        1 - \exp \bigg( -\sum_{j \in M} P_j(i) \bigg) 
        \ge 
        \big(1 - e^{-1} \big) \sum_{j \in M} P_j(i)
        \ge
        \half \sum_{j \in M} P_j(i)~,
    \]
    using the inequality $1-x \le e^{-x}$
    and $1-e^{-x}\ge \left(1-e^{-1}\right)x$ that holds for any $x \in [0,1]$.
    The probability that $Q$ is thrown into a nonempty bin is at least that of it being thrown into a nonempty bin $i \in S$. This is is at least
    \[
        \sum_{i \in S} Q(i) \cdot \half \sum_{j \in M} P_j(i)
        =
        \half \sum_{j \in M} \sum_{i \in S} Q(i) P_j(i)~,
    \]
    where $\sum_{i \in S} Q(i) P_j(i)$ is the probability that both $Q$ and $P_j$ end up in to same bin in $S$. As $\TV{P_j}{Q} \le 1/4$ and $Q(S) \ge 1/2$, \cref{lem:coalesceonsubset} implies that the latter probability is at least $1/16n$.
    Therefore, the probability of that the additional ball is thrown into a nonempty bin is at least
    \[
        \half \lvert M \rvert \cdot \frac{1}{16n}
        =
        \frac{\vert M \rvert}{32n}~. 
    \]
\end{proof}

\begin{lemma}
    \label{lem:coalescenceeasycase}
    Suppose with first throw a set of balls $M$ with probability distributions $P_j$ over the bins for every $j \in M$. Thereafter, we throw an additional ball with probability distribution $Q$ such that $\lvert Q - P_j \rvert \le 1/4$ for all $j \in M$. Additionally, denote
    \[
        S = \{i \in [n] : \sum_{j \in M} P_j(i) \le 1\}~,
    \]
    and suppose that $Q(S) < 1/2$. 
    Then, the probability that $Q$ is thrown into a nonempty bin is at least $1/4$.
\end{lemma}

\begin{proof}
    The probability of bin $i \not\in S$ not being empty is 
    \[
        1 - \prod_{j \in M} \big(1 - P_j(i) \big)
        \ge
        1 - \exp \bigg( -\sum_{j \in M} P_j(i) \bigg)
        \ge
        1 - \exp(-1)
        \ge 
        \half~.
    \]
    The probability that $Q$ is thrown into a nonempty bin is at least its probability of it being thrown into a nonempty bin in $S^c$ which is exactly 
    \[
        \half Q(S^c) 
        \ge 
        \half \cdot \half
        = 
        \frac{1}{4}~. 
    \]
\end{proof}

\clearpage
\section{Proofs for Section 4}
\subsection{Multiplicative weights}
We begin with a classic result on the Hedge algorithm. 

\begin{algorithm}
\caption{Hedge}
\label{alg:Hedge}
\begin{algorithmic}[1]
\STATE Input: number of experts $k$, number of iterations $T$.
\STATE Let $\beta= \sqrt{\frac{\log k}{T}}$
\STATE Initialize $W^{(1)}(i) = 1$, for $i = 1,\ldots,k$.
\FOR{$t = 1,\ldots,T$}
    \STATE Set $w^{(t)}(i) = \frac{W^{(t)}(i)}{\sum_{i=1}^k W^{(t)}(i)}$, for $i = 1,\ldots,k$.
    \STATE Observe $c_t(i)$ , for $i = 1,\ldots,k$.
    \STATE Incur loss $\sum_{i=1}^k w^{(t)}(i)c_t(i_t)$
    \STATE Update weights $W^{(t+1)}(i) = W^{(t)}(i) \cdot \exp\left( - \beta  c_t (i) \right)$,  $\forall i \in  [1,..,k]$. 
\ENDFOR
\end{algorithmic}
\end{algorithm}

\begin{theorem}[\citep{freund1997decision}]
\label{theorem_MW}
Assume that  $0\le c_t(i)\le 1$ for all $t=1,\ldots,T$. Hedge (\cref{alg:Hedge}) satisfies that for any strategy $w \in \Delta _k$:

$$
\sum w_t\cdot c_t -  \sum w\cdot c_t\le 
2 \sqrt{ T \log k}.
$$
\end{theorem}

Note that in \cref{alg:mwal} and in \cref{alg:mwal2}, we actually run the Hedge algorithm with the estimates $\tilde g_t(i)$ as the costs $c_t(i).$  
We obtain $\tilde g_t(i)$ by shifting and scaling   $ g_t(i)$, so that  $\tilde g_t(i) \in [0, 1]$
and we can 
 apply  \cref{theorem_MW}. 

\begin{corollary}
\label{theorem_MW2}
Let $-B \le g_t(i)\le B$, and $\tilde g_t(i) = (g_t(i)+B)/2B$. Assume that we run the Hedge algorithm with costs
$c_t(i)$ equal to $\tilde g_t(i)$. We have that
$$
\frac{1}{T} \left( \sum w_t\cdot g_t - \min_{w\in \Delta_k} \sum w\cdot g_t\right) \le 
4B \sqrt{ \frac{\log k}{T}},
$$

\end{corollary}
\begin{proof}

The losses $\tilde g_t(i)$ satisfy the conditions of \cref{theorem_MW}. Therefore, 
$$
\sum w_t\cdot \tilde g_t - \min_{w\in \Delta_k} \sum w\cdot \tilde g_t \le 
2 \sqrt{ T \log k}.
$$
This implies that
$$
\sum w_t\cdot  (g_t+B\textbf{1})/2B - \min_{w\in \Delta_k} \sum w\cdot (g_t+B\textbf{1})/2B
\le 
2 \sqrt{ T \log k},
$$
where $\textbf{1}$ denotes a vector of ones. Multiplying by $2B$ gives
$$
\sum w_t\cdot  (g_t+B\textbf{1}) - \min_{w\in \Delta_k} \sum w\cdot (g_t+B\textbf{1}) 
\le 
4B \sqrt{ T \log k}.
$$
Observing that $\forall w \in \Delta_k,$ $w\cdot B\textbf{1}=B$ we get that
$$
\sum w_t\cdot  g_t +B - \min_{w\in \Delta_k} \sum w\cdot g_t -B 
\le 
4B \sqrt{ T \log k}
$$
as stated.
\end{proof}

\subsection{Estimating the feature expectations of the expert }
\label{sec:thm8}
We begin this subsection with \cref{lemma:sampling} that bounds the number of samples needed from the expert in order to get a good approximation of the expectations of its features.

\begin{lemma}
\label{lemma:sampling}
For any $\epsilon,\delta$, given $m\ge \frac{2\ln(2k/\delta)}{\epsilon^2}$ samples from the stationary distribution $\pi^E$, with probability at least $1-\delta$, the approximate feature expectations $\Hat{\Phi}_E$ satisfy that  $\|\Hat{\Phi}_E-\Phi_E \|_\infty \le \epsilon.$ 
\end{lemma}
\begin{proof}
By Hoeffding's inequality we get that
\begin{align*}
&  \forall i \in [1,..,k] \enspace  \text{Pr}(|\Hat{\Phi}_E(i)-\Phi_E(i)| \ge \epsilon) \le 2\exp (-m\epsilon^2/2).\\
\shortintertext{Applying the union bound over the features we get that 
}
&  \text{Pr}(\exists i \in [1,..,k], s.t., |\Hat{\Phi}_E(i)-\Phi_E(i)| \ge \epsilon) \le 2k\exp (-m\epsilon^2/2).\\
\shortintertext{This is equivalent to }
& \text{Pr}(\forall i \in [1,..,k] \enspace |\Hat{\Phi}_E(i)-\Phi_E(i)| \le \epsilon) \ge 1- 2k\exp (-m\epsilon^2/2).\\
\shortintertext{and to}
& \text{Pr}(\|\Hat{\Phi}_E-\Phi_E\|_\infty \le \epsilon) \ge 1- 2k\exp (-m\epsilon^2/2).
\end{align*}

The Lemma now follows by substituting the value of $m$.
\end{proof}

\begin{theorem*}[\ref{theorem:MW_knownp}]
Assume we run \cref{alg:mwal} for $T =  \frac{144\log k}{\epsilon^2} $ iterations,
using $m = \frac{18\log (2k/\delta)}{\epsilon^2} $ samples from $\mu(\pi^E)$.
Let $\bar\psi$ be the mixed policy returned by the algorithm. 
Let $v^\star$ be the game value as in \cref{eq:objective}.
Then, we have that
$
    \rho(\bar\psi) - \rho(\pi^E)
    \ge 
    v^\star - \epsilon
$ 
with probability at least $1-\delta$,
where $\rho$ is the average of any reward of the form 
$r(s) = w\cdot \phi(s)$ where $w\in\Delta_k$.
\end{theorem*}

\begin{proof}

 \cref{theorem_MW2} with $B=1$ and $T = \frac{144\log k}{\epsilon^2}$ gives that 
\begin{equation}
\frac{1}{T} \left( \sum w_t\cdot g_t - \min_{w\in \Delta_k} \sum w\cdot g_t\right) \le \frac{\epsilon}{3},
\label{eq:17}
\end{equation}
where $g_t(i) = \Phi(\pi^{(t)})[i] - \wt \Phi^E[i]$. Note also that  \cref{lemma:sampling} with $m = \frac{18\log (2k/\delta)}{\epsilon^2} $ gives that $\|\Hat{\Phi}_E-\Phi_E \|_\infty \le \frac{\epsilon}{3},$ which implies that, for any $w\in \Delta_k$:
\begin{equation}
\label{eq:18}
w\cdot \Hat \Phi_E \le w \cdot \Phi_E + \epsilon/3,
\end{equation}
and
\begin{equation}
\label{eq:18_1}
w\cdot \Phi_E \le w \cdot \Hat{\Phi}_E + \epsilon/3,
\end{equation}

\allowdisplaybreaks

Now, let $\Bar{w}=\frac{1}{T}\sum_{t=1}^T w^{(t)}$, and recall that $\bar \psi$ is the mixed policy that assigns probability $\frac{1}{T}$ to $\pi^{(t)}$ for all $t \in \{1,\ldots, T\}$. Thus,
\begin{align*}
v^\star & =   \max_{\psi\in\Psi} \min_{w\in\Delta_k} \left[ w\cdot\Phi(\psi)-w\cdot\Phi_E\right] \\
& =  \min_{w\in\Delta_k} \max_{\psi\in\Psi} \left[ w\cdot\Phi(\psi)-w\cdot\Phi_E\right] \tag{von Neumann's minimax theorem}   \\
& \le  \min_{w\in\Delta_k} \max_{\psi\in\Psi} \left[ w\cdot\Phi(\psi)-w\cdot\hat{\Phi}_E\right]+\epsilon/3 \tag{\cref{eq:18}}   \\
& \le  \max_{\psi\in\Psi} \left[ \bar{w}\cdot\Phi(\psi)-\bar{w}\cdot\hat{\Phi}_E\right]+\epsilon/3 \\
& =  \max_{\psi\in\Psi} \frac{1}{T}\sum\nolimits_{t=1}^T\left[ w^{(t)}\cdot\Phi(\psi)-w^{(t)}\cdot\hat{\Phi}_E\right]+\epsilon/3 \tag{Definition of $\bar w$}   \\
& \le  \frac{1}{T}\sum\nolimits_{t=1}^T\max_{\psi\in\Psi} \left[ w^{(t)}\cdot\Phi(\psi)-w^{(t)}\cdot\hat{\Phi}_E\right]+\epsilon/3 \\
& =  \frac{1}{T}\sum\nolimits_{t=1}^T \left[ w^{(t)}\cdot\Phi(\pi^{(t)})-w^{(t)}\cdot\hat{\Phi}_E\right]+ \epsilon/3 \tag{$\pi^{(t)}$ is optimal w.r.t the reward  $w^{(t)}$}   \\
& \le  \frac{1}{T}\min_{w\in\Delta_k}\sum\nolimits_{t=1}^T \left[ w\cdot\Phi(\pi^{(t)})-w\cdot\hat{\Phi}_E\right]+ 2\epsilon/3 \tag{\cref{eq:17}}   \\
& =  \min_{w\in\Delta_k} \left[ w\cdot\Phi(\bar\psi)-w\cdot\hat{\Phi}_E\right]+2\epsilon/3 && \tag{Definition of $\bar \psi$}   \\
& \le  \min_{w\in\Delta_k} \left[ w\cdot\Phi(\bar\psi)-w\cdot\Phi_E\right]+\epsilon \tag{\cref{eq:18_1}} \\
& \le   w^\star\cdot\Phi(\bar\psi)-w^\star\cdot\Phi_E+\epsilon && \tag{For any $w^*\in\Delta_k$} \\
& =   \rho(\bar\psi)-\rho(\pi^E)+\epsilon. 
\end{align*}
\end{proof}

\subsection{Estimating the game matrix directly}
\label{sec:thm9}
In this section we prove \cref{MW_unbiased_loss}.
Our proof uses the following version of Azuma's concentration bound.

\begin{lemma}[Azuma inequality]
Let $\{y_t \}_{t=1}^T$ be a sequence of random variables such that $-b \le y_t \le b$, for $1 \le t < T$.  Let $E_t =y_t - \mathbb{E}[y_t\mid y_1, ..., y_{t-1}]$ be the martingale difference sequence defined over the sequence  $\{y_t \}_{t=1}^T$. Then
$$\Pr \left(\left\lvert\frac{1}{T}\sum _{t=1}^T E_t\right\lvert  \ge \epsilon\right) \le 2\exp\left(-\frac{T\epsilon^2}{8b^2}\right)$$
\end{lemma}

\begin{theorem*}[Restatement of \cref{MW_unbiased_loss}]
    Assume we run \cref{alg:mwal2} for $T $ iterations, and
there exists a parameter $b$, such that for any $\ell$, $\Pr(\| g_t\|_\infty \ge \ell\cdot b) \le e^{-\ell}$.
 Let $\bar\psi$ be the mixed policy returned by the algorithm. 
Let $v^\star$ be the game value as in \cref{eq:objective}.
Then, there exists a constant $c$ such that for $T \ge c B\log^2 B$ where $B=\frac{b^2\log^3 k\log^2(1/\delta)}{\epsilon^2}$, we have that
$
    \rho(\bar\psi) - \rho(\pi^E)
    \ge 
    v^\star - \epsilon
$ 
with probability at least $1-\delta$,
where $\rho$ is the average of any reward of the form 
$r(s) = w\cdot \phi(s)$ where $w\in\Delta_k$. 
\end{theorem*}
\begin{proof}
Let     $\ell = \max \{ \log\left(\frac{T}{\delta}\right), \log \left( \frac{1}{\epsilon} \right) \}$.
Then for any $t$ we have that 
$\Pr(\| g_t\|_\infty \ge \ell\cdot b) \le \frac{\delta}{T}$.
By the union bound it follows that with probability $1-\delta$ 
for all times $t=1,\ldots, T$, we  have that
$\| g_t\|_\infty \le \ell b$.
We denote by ${\cal F }$ the subspace of 
our probability space that includes all runs of the algorithm in which $\| g_t\|_\infty \le \ell b$ for all 
$t=1,\ldots, T$. We have that at least 
$1-\delta$ fraction of the runs of the algorithm are
in ${\cal F }$.

By the definition of $g_t$
we have that  $\mathbb{E} \left[ g_t | g_{1},...,g_{t-1} \right] = \Phi(\pi^{(t)}) -\Phi_E$. Furthermore $w^{(t)}$ depends only on
$ g_{1},...,g_{t-1}$ and not on $g_t$.
It follows that the  random variables 
$E_t=w^{(t)} \cdot g_t - E[w^{(t)} \cdot g_t \mid g_{1},...,g_{t-1}] =
 w^{(t)} \cdot \left( g_t - (\Phi(\pi^{(t)}) -\Phi_E) \right)$
 is a martingale difference sequence.
We would like to apply Azuma's inequality to this sequence, but the difficulty is that the variables $E_t$ are unbounded. 

To deal with this problem we define new variables
$\bar g_t$ as follows
\[
\bar g_t = 
     \begin{cases}
g_t &\quad \| g_t(i) \|_\infty \le \ell b , \\
0 &\quad \text{otherwise} \ ,
\end{cases}
\]
and we define the  martingale difference sequence 
$\bar E_t=w^{(t)} \cdot \bar g_t - E[w^{(t)} \cdot \bar g_t \mid g_{1},...,g_{t-1}]$.
Unfortunately, $E[w^{(t)} \cdot \bar g_t \mid g_{1},...,g_{t-1}]$
does not equal to $\Phi(\pi^{(t)}) -\Phi_E$. 
But we can bound the difference as follows.
\begin{align}
| E[w^{(t)} \cdot  g_t \mid g_{1},...,g_{t-1}] & - 
E[w^{(t)} \cdot \bar g_t \mid g_{1},...,g_{t-1}]| \nonumber \\
& \le \int_{x=\ell b}^\infty 
\Pr \left(w^{(t)}\cdot g_t  >x \right) {\rm dx}
- \int_{x=-\ell b}^\infty 
\Pr \left(w^{(t)}\cdot g_t  <x \right) {\rm dx} \nonumber \\
&
\le
 \int_{x=\ell b}^\infty
\Pr\left(\| g_t\|_\infty \ge x \right) {\rm dx} \nonumber \\
 &=
 \int_{x=\ell}^\infty
 \Pr\left(\| g_t\|_\infty \ge xb \right) {\rm dx} \nonumber \\
 &\le  
 \int_{x=\ell}^\infty 
 e^{-x}{\rm dx}  = e^{-\ell} \le \epsilon \ , \label{eq:exp-diff}
\end{align}

where the first inequality follows from the formula $E(Y) = \int_{x=0}^\infty \Pr(Y>x) - \int_{x=0}^{-\infty} \Pr(Y<x)$ (which is  derived from the more familiar formula $E(Y) = \int_{x=0}^\infty \Pr(Y>x)$ for a nonnegative variable $Y$). The second inequality follows since $w\in \Delta_k$ and the last equality follows by the definition of $\ell$.
By applying Azuma's inequality to the sequence $\bar E_t$ we
get that 
\[
\Pr \left(\left\lvert\frac{1}{T}\sum _{t=1}^T \bar E_t\right\lvert  \ge \epsilon\right) \le 2\exp\left(-\frac{T\epsilon^2}{8(\ell b)^2}\right) \ .
\]

Our choice of $T$ guarantees that 
\[
2\exp\left(-\frac{T\epsilon^2}{8(\ell b)^2}\right) \le \delta \ .
\]

So we also have that within the subspace ${\cal F}$
\begin{equation} \label{eq:ebar}
\Pr_{{\cal F}} \left(\left\lvert\frac{1}{T}\sum _{t=1}^T \bar E_t\right\lvert  \ge \epsilon\right) \le \frac{\delta}{1-\delta} \ .
\end{equation}

But in ${\cal F}$, $\bar g_t = g_t$ and therefore 
$\bar E_t = E_t - E[w^{(t)} \cdot  g_t \mid g_{1},...,g_{t-1}] + 
E[w^{(t)} \cdot \bar g_t \mid g_{1},...,g_{t-1}]$.
So by \cref{eq:exp-diff}),
\begin{equation} \label{eq:ebarande}
| E_t - \bar E_t | \le \epsilon \ .
\end{equation}
 It follows from Equations (\ref{eq:ebar}) and (\ref{eq:ebarande})
that within ${\cal F}$:
\begin{equation} \label{eq:finalE}
\Pr_{\cal F} \left(\left\lvert\frac{1}{T}\sum _{t=1}^T  E_t\right\lvert  \ge 2\epsilon \right) \le \frac{\delta}{1-\delta}\ .
\end{equation}

Let $\Bar{w}=\frac{1}{T}\sum_{t=1}^T w^{(t)}$,  and recall that $\bar \psi$ is the mixed policy that assigns probability $\frac{1}{T}$ to $\pi^{(t)}$ for all $t \in \{1,\ldots, T\}$.
We have that

\begin{align}
v^\star & =   \max_{\psi\in\Psi} \min_{w\in\Delta_k} \left[ w\cdot\Phi(\psi)-w\cdot\Phi_E\right] && \qquad\qquad \nonumber \\
& =  \min_{w\in\Delta_k} \max_{\psi\in\Psi} \left[ w\cdot\Phi(\psi)-w\cdot\Phi_E\right]&& \qquad\qquad \text{von Neumann's minimax theorem}  \nonumber \\
& \le  \max_{\psi\in\Psi} \left[ \bar{w}\cdot\Phi(\psi)-\bar{w}\cdot{\Phi}_E\right] &&  \qquad\qquad  \nonumber \\
& =  \max_{\psi\in\Psi} \frac{1}{T}\sum_{t=1}^T\left[ w^{(t)}\cdot\Phi(\psi)-w^{(t)}\cdot\Phi_E\right]&& \qquad\qquad \text{Definition of $\bar w$} \nonumber \\
& \le  \frac{1}{T}\sum_{t=1}^T\max_{\psi\in\Psi} \left[ w^{(t)}\cdot\Phi(\psi)-w^{(t)}\cdot\Phi_E\right]  &&  \qquad\qquad   \nonumber \\
& =  \frac{1}{T}\sum_{t=1}^T \left[ w^{(t)}\cdot\Phi(\pi^{(t)})-w^{(t)}\cdot\Phi_E\right] && \qquad\qquad \text{$\pi^{(t)}$ is optimal w.r.t the reward  $w^{(t)}$}  \label{eq:lastexp} 
\end{align}

Now we continue our derivation assuming that the run of the algorithm is in ${\cal F}$.
 We
use Equation (\ref{eq:finalE}) and say that with probability 
$1-\frac{\delta}{1-\delta}$ 
the  expression in  (\ref{eq:lastexp}) is bounded by
\begin{equation} \label{eq:2epsilon}
  \frac{1}{T}\sum_{t=1}^T  w^{(t)}\cdot g_t + 2\epsilon \ .
\end{equation}
Our choice of $T$ also guarantees that
\[
4\ell b \sqrt{\frac{\log k}{T}} \le \epsilon \ .
\]
and therefore for a run in ${\cal F}$,  the bound on the regret of Hedge in Corollary \ref{theorem_MW2} implies that expression in
(\ref{eq:2epsilon}) is bounded by
\begin{equation} \label{eq:3epsilon}
  \frac{1}{T}\min_{w\in\Delta_k}\sum_{t=1}^T  w\cdot g_t+3\epsilon 
\end{equation}

Let $w_{\rm min}\in\Delta_k$ be the vector achieving the 
minimum in Equation (\ref{eq:3epsilon}).
To finish the proof we need to bound Equation (\ref{eq:3epsilon}) with
\begin{equation} 
  \frac{1}{T}\sum_{t=1}^T w_{\rm min}\cdot \left( g_t - (\Phi(\pi^{(t)}) -\Phi_E) \right)\ .
\end{equation}

For this we would like to apply Azuma's inequality to each of the $k$
martingale differences sequences 
$ X_t(i)= g_t(i) - \mathbb{E}[g_t(i)\mid g_1,\ldots,g_{t-1}] = g_t(i) - (\Phi(\pi^{(t)}[i] -\Phi_E[i])$. 
As before, since the $g_t(i)$'s are unbounded we look instead at the martingale sequence 
$ \bar X_t(i) = \bar{g}_t(i) - \mathbb{E}[\bar{g}_t(i)\mid g_1,\ldots,g_{t-1}]. $

Unfortunately, as before, $ \mathbb{E}[\bar{g}_t(i)\mid g_1,\ldots,g_{t-1}]$ does not equal to $ \mathbb{E}[g_t(i)\mid g_1,\ldots,g_{t-1}]$. 
But we can bound the difference as follows.
\begin{align}
| \mathbb{E}[ g_t(i)-\bar g_t(i) \mid g_{1},...,g_{t-1}]| 
& \le \int_{x=\ell b}^\infty 
\Pr \left(g_t(i)  >x \right) {\rm dx}
- \int_{x=-\ell b}^\infty 
\Pr \left(g_t(i)  <x \right) {\rm dx} \nonumber \\
&
\le
 \int_{x=\ell b}^\infty
\Pr\left(\| g_t\|_\infty \ge x \right) {\rm dx} 
 =
 \int_{x=\ell}^\infty
 \Pr\left(\| g_t\|_\infty \ge xb \right) {\rm dx} \nonumber \\
 &\le  
 \int_{x=\ell}^\infty 
 e^{-x}{\rm dx}  = e^{-\ell} \le \epsilon \ , \label{eq:exp-diff2}
\end{align}

where the inequalities follow from the same reasons as  in \cref{eq:exp-diff}. 

By applying Azuma's inequality to the sequence $\bar X_t(i)$ we
get that 
$
\Pr \left(\left\lvert\frac{1}{T}\sum _{t=1}^T \bar X_t(i)\right\lvert  \ge \epsilon\right) \le 2\exp\left(-\frac{T\epsilon^2}{8(\ell b)^2}\right)  ,
$
and our choice of $T$ guarantees that 
$
2\exp\left(-\frac{T\epsilon^2}{8(\ell b)^2}\right) \le \frac{\delta}{k}.
$
So we also have that within the subspace ${\cal F}$
\begin{equation} \label{eq:ebar2}
\Pr_{{\cal F}} \left(\left\lvert\frac{1}{T}\sum _{t=1}^T \bar X_t(i)\right\lvert  \ge \epsilon\right) \le \frac{\delta}{k(1-\delta)} \ .
\end{equation}

But in ${\cal F}$, $\bar g_t(i) = g_t(i)$ and therefore 
$\bar X_t(i) = X_t(i) - \mathbb{E}[g_t(i) \mid g_{1},...,g_{t-1}] + 
\mathbb{E}[\bar g_t(i) \mid g_{1},...,g_{t-1}]$.
So by \cref{eq:exp-diff2}),
\begin{equation} \label{eq:ebarande2}
| X_t(i) - \bar X_t(i) | \le \epsilon \ .
\end{equation}
 It follows from Equations (\ref{eq:ebar2}) and (\ref{eq:ebarande2})
that within ${\cal F}$:
\begin{equation} \label{eq:finalE2}
\Pr_{\cal F} \left(\left\lvert\frac{1}{T}\sum _{t=1}^T  X_t(i)\right\lvert  \ge 2\epsilon \right) \le \frac{\delta}{k(1-\delta)}\ .
\end{equation}

By applying the union bound over the features we get that 
\begin{align}
&  \Pr_{\cal F}\left(\exists i \in [1,..,k], s.t., \left\lvert \frac{1}{T}\sum _{t=1}^T  X_t(i) \right\rvert \ge 2\epsilon\right) \le \frac{\delta}{1-\delta}.\nonumber\\
\shortintertext{This is equivalent to }
& \Pr_{\cal F}\left(\forall i \in [1,..,k] \enspace \left\lvert \frac{1}{T}\sum _{t=1}^T  X_t(i) \right\rvert \le 2\epsilon\right) \ge 1- \frac{\delta}{1-\delta}.
\label{eq:union2}
\end{align}

Equation (\ref{eq:union2}) implies that with probability 
$1-\frac{\delta}{1-\delta}$ in ${\cal F}$, for any $w\in \Delta_k$ it holds that:
\begin{equation*}
\frac{1}{T}\sum_{t=1}^T w \cdot \left( g_t - (\Phi(\pi^{(t)}) -\Phi_E) \right) \le 2\epsilon\ .
\end{equation*}
Since it is true for any $w$, we get that that we can upper bound Equation (\ref{eq:3epsilon}) by 
\begin{equation} \label{eq:eprimeuse}
 \frac{1}{T}\min_{w\in\Delta_k}\sum_{t=1}^T \left[ w\cdot\Phi(\pi^{(t)})-w\cdot\Phi_E\right]+ 5\epsilon \ .
\end{equation}

The theorem now follows\footnote{We have to  scale down $\epsilon$ by $5$. We also have to scale down $\delta$ by $3$ since our bound fails to hold with probability $3\delta$. Indeed, with probabilty $\le \delta$ our run is not in ${\cal F}$, and with probability
$1-\delta$ it is in  ${\cal F}$, and either of the bounds in Equation (\ref{eq:2epsilon}) and (\ref{eq:eprimeuse}) fails -- which happens with probability $\le \frac{2\delta}{1-\delta}$.}  since the expression in the last equation is smaller than
$\rho(\bar\psi)-\rho(\pi^E)+5\epsilon$ where
$\rho$ is the average reward of the form $r(s)=w\phi(s)$ for
any $w\in \Delta_k$.
\end{proof}
\newpage

\newpage
\end{appendices}

\end{document}